\documentclass[9pt,twocolumn,letterpaper]{article}

\usepackage{wacv}
\usepackage{times}
\usepackage{epsfig}
\usepackage{graphicx}
\usepackage{amsmath}
\usepackage{amssymb}
\usepackage{algorithm}
\usepackage{algpseudocode}

\usepackage{booktabs, adjustbox, multirow, tabularx, subfig}
\usepackage{amsthm,amsfonts}

\def\wacvPaperID{1264} 

\wacvfinalcopy 

\ifwacvfinal
\fi

\ifwacvfinal
\usepackage[breaklinks=true,bookmarks=false]{hyperref}
\else
\usepackage[pagebackref=true,breaklinks=true,colorlinks,bookmarks=false]{hyperref}
\fi

\begin{document}

\title{Nonnegative Low-Rank Tensor Completion via Dual Formulation with Applications to Image and Video Completion}

\author{Tanmay Kumar Sinha\thanks{Equal contribution}\\
IIIT, Hyderabad\\
{\tt\small tanmay.kumar@research.iiit.ac.in}
\and
Jayadev Naram\footnotemark[1]\\
IIIT, Hyderabad\\
{\tt\small jayadev.naram@research.iiit.ac.in}
\and
Pawan Kumar\\
IIIT, Hyderabad\\
{\tt\small pawan.kumar@iiit.ac.in}
}


\maketitle
\thispagestyle{empty}

\theoremstyle{plain}
\newtheorem{theorem}{Theorem}
\newtheorem{corollary}[theorem]{Corollary}
\newtheorem{lemma}[theorem]{Lemma}
\newtheorem{prop}[theorem]{Proposition}
\newtheorem{assume}{Assumption}
\newtheorem{remark}[theorem]{Remark}

\theoremstyle{definition}
\newtheorem{mydef}[theorem]{Definition}
\newtheorem{example}[theorem]{Example}

\newcommand{\Norm}[1]{\left\lVert#1\right\rVert}

\newcommand{\R}{\mathbb{R}}
\newcommand{\W}{\mathcal{W}}
\newcommand{\Y}{\mathcal{Y}}
\newcommand{\Z}{\mathcal{Z}}
\newcommand{\s}{\mathcal{S}}
\newcommand{\Rtensor}{\R^{n_1\times \cdots \times n_K}}
\newcommand{\Rptensor}{\R_+^{n_1\times \cdots \times n_K}}
\newcommand{\A}{\mathcal{A}}
\newcommand{\M}{\mathcal{M}}
\newcommand{\N}{\mathcal{N}}
\newcommand{\h}{\mathcal{H}}
\newcommand{\T}{\mathcal{T}}
\newcommand{\Prob}{\mathbb{P}}
\newcommand{\Dist}{\mathcal{D}}
\newcommand{\perpProj}{\mathcal{P}^\perp}
\newcommand{\bb}{\mathbb{B}}
\newcommand{\Sprod}{\mathbb{S}_{xy}}
\newcommand{\highlight}[1]{\textsl{\textbf{#1}}}
\newcommand{\mapping}[3]{#1:#2\rightarrow #3}
\newcommand{\doubt}{\highlight{[??]}}
\newcommand{\bigvert}[2]{\left.#1\right|_{#2}}
\newcommand{\sdnn}[1]{${#1}$}
\newcommand{\bsdnn}[1]{$\boldsymbol{#1}$}
\newcommand{\ifthen}[2]{\textbf{(#1)}\boldsymbol{\implies}\textbf{(#2)}}
\newcommand{\bsdn}[1]{\boldsymbol{#1}}
\newcommand{\forward}{$(\implies)$}
\newcommand{\converse}{$(\impliedby)$}
\newcommand{\Lt}[1]{\underset{#1\rightarrow 0}{Lt}}
\newcommand{\norm}[1]{\|#1\|}
\newcommand{\dparder}[2]{\dfrac{\partial #1}{\partial x_{#2}}}
\newcommand{\fparder}[2]{\frac{\partial #1}{\partial x_{#2}}}
\newcommand{\parder}[2]{\partial #1/\partial x_{#2}}
\newcommand{\parop}[1]{\dfrac{\partial}{\partial x_{#1}}}
\newcommand{\innerproduct}[2]{\langle #1, #2 \rangle}
\newcommand{\genst}{St_B(n,p)}
\newcommand{\igenst}[1]{St_{B_{#1}}(n_{#1},p)}
\newcommand{\realmat}[2]{\R^{#1\times #2}}
\newcommand{\Skew}{\mathcal{S}_{skew}(p)}
\newcommand{\Sym}{\mathcal{S}_{sym}(p)}
\newcommand{\XperpB}{X_{B^\perp}}
\newcommand{\polarRetr}{R^{polar}_X}
\newcommand{\qrRetr}{R^{QR}_X}
\newcommand{\vectransport}{\mathcal{T}}
\newcommand{\grad}{\text{grad}\,}
\newcommand{\hess}{\text{Hess}\,}
\newcommand{\unfold}[1]{\textit{unfold}_{#1}}
\newcommand{\fold}[1]{\textit{fold}_{#1}}

\begin{abstract}
Recent approaches to the tensor completion problem have often overlooked the nonnegative structure of the data. We consider the problem of learning a nonnegative low-rank tensor, and using duality theory, we propose a novel factorization of such tensors. The factorization decouples the nonnegative constraints from the low-rank constraints. 
The resulting problem is an optimization problem on manifolds, and we propose a variant of Riemannian conjugate gradients to solve it.
We test the proposed algorithm across various tasks such as colour image inpainting, video completion, and hyperspectral image completion. Experimental results show that the proposed method outperforms many state-of-the-art tensor completion algorithms.


\end{abstract}

\section{Introduction}
Recent years have seen an increase in the quantity of multidimensional data available, such as colour images, video sequences, and 3D images. Flattening multidimensional data to matrices usually leads to loss of information as matrices cannot capture the inherent structures present in most multidimensional data. This has led to increased research on tensor-based techniques for handling such data. 

The low-rank tensor completion problem aims to recover an original tensor from partial observations. A well-known \cite{dual} formulation for such problems is 
\begin{alignat}{3}
\underset{\W\in\Rtensor}{\text{min}}& &&C\,L(\W,\Y_\Omega) + R(\W), \label{eqn:general_tensor_problem}
\end{alignat}
where $\Y_\Omega\in \Rtensor$ is a partially observed tensor for indices given in the set $\Omega$, $\mapping{L}{\Rtensor}{\R}$ is a loss function, $C>0$ denotes the cost parameter, and $R$ is a regularizer enforcing low-rank constraint. 

In many applications of tensor reconstruction such as color image recovery, video completion, recommendation systems, and link prediction, the data is nonnegative. Problem \eqref{eqn:general_tensor_problem} does not enforce this structural constraint, and as such, the recovered tensors might contain negative entries. 

To incorporate these constraints, we consider the nonnegative low-rank tensor learning problem of the form:
\begin{equation}
\begin{aligned}
\underset{\W\in\Rtensor}{\min}& &&C\|\W_\Omega-\Y_\Omega\|^2 + R(\W)\\
\text{subject to}& &&\W \ge 0, \label{eqn:primal_problem}
\end{aligned}
\end{equation}
where $(\W_\Omega)_{i_1,\ldots,i_K} = \W_{i_1,\ldots,i_K}$ if $({i_1,\ldots,i_K})\in \Omega$.
We convert the problem \eqref{eqn:primal_problem} into a minimax problem by constructing a partial dual similar to \cite{structured_matrix_completion}. This leads to a factorization of the tensor $\W$ in a form with separate factors for the nonnegative and low-rank constraints. The minimax problem has a rich geometric structure. We employ a Riemannian conjugate gradient algorithm to exploit this structure and develop an efficient solution.

The main contributions of the paper are listed below.
\begin{itemize}
    \item We propose a novel factorization for modeling nonnegative low-rank tensors.
    \item We develop an algorithm exploiting the inherent geometric structure of this factorization.
    \item Experiments carried out on several real-world datasets show that the proposed algorithm outperforms state-of-the-art tensor completion algorithms.
\end{itemize}
The rest of the paper is organized as follows. In Section 2, we introduce the notation used in the paper. In Section 3, we review previous work related to the tensor completion problem. In Sections 4 and 5, we develop the dual framework and present our algorithm. Section 6 details experiments carried out to compare our algorithm with several state-of-the-art algorithms. In Section 7, we end with concluding remarks.

\section{Notation}

For a full treatment of tensors, we refer to \cite{kolda_review}. Here, we outline the basic notation we use for tensors. We denote tensors by calligraphic capital letters and matrices by capital letters. 
For a matrix $X\in\R^{m\times n}$, the nuclear norm of $X$, denoted by $\|X\|_*$, is the $l_1$-norm of the singular values of $X$. The inner product of two same-sized tensors $\W,\mathcal{U}\in \Rtensor$ is the sum of the products of their entries
\begin{equation*}
    \innerproduct{\W}{\mathcal{U}} = \sum_{i_1}^{n_1}\sum_{i_2}^{n_2}\cdots \sum_{i_K}^{n_K} \W_{i_1,\ldots,i_K}\mathcal{U}_{i_1,\ldots,i_K}.
\end{equation*}
A mode-$k$ fiber of a tensor $\W\in\Rtensor$, denoted by $\W_{i_1,\ldots,i_{k-1},:,i_{k+1},\ldots,i_K}$, is a vector obtained by fixing all but $k$-th index of $\W$. The mode-$k$ unfolding of a tensor $\W\in\Rtensor$ is a matrix $W_k\in \R^{n_k\times n_1\ldots n_{k-1}n_{k+1}\ldots n_K}$ formed by arranging the mode-$k$ fibers to be the columns of the resulting matrix, i.e., 
\begin{equation*}
W_k = [\W_{i_1,\ldots,i_{k-1},:,i_{k+1},\ldots,i_K}]\;\forall i_j, j \ne k.
\end{equation*}
The reverse of unfolding operation is called the folding operation which converts a given matrix to a tensor of specific order. We also represent the mode-$k$ unfolding by the map $\mapping{\unfold{k}}{\Rtensor}{\R^{n_k\times n_1\ldots n_{k-1}n_{k+1}\ldots n_K}}$ such that $\unfold{k}(\W) = W_k$, and the mode-$k$ folding by the map $\mapping{\fold{k}}{\R^{n_k\times n_1\ldots n_{k-1}n_{k+1}\ldots n_K}}{\Rtensor}$. The $k$-mode product of a tensor $\W\in\Rtensor$ with a matrix $X\in \R^{m\times n_k}$ is denoted by $\W \times_k X\in \R^{n_1 \times\ldots\times n_{k-1}\times m\times n_{k+1}\times \ldots n_K}$, defined element-wise as follows:
\begin{equation*}
    (\W\times_k X)_{i_1,\ldots,i_{k-1},j,i_{k+1},\ldots,i_K} = \sum_{i_k}^{n_k} \W_{i_1,\ldots,i_K}X_{j,i_k}.
\end{equation*}
Then we have 
\begin{equation*}
\mathcal{U} = \W\times_k X \Longleftrightarrow U_k = XW_k.
\end{equation*}

\section{Previous Work}
Tensor completion for visual data recovery was introduced in \cite{oldest_lrtc}, building on the framework for low-rank matrix completion using matrix trace norm regularizer. The trace norm for tensors can be defined in several ways, and as such there exist multiple formulations for trace norm regularized tensor completion. \cite{oldest_lrtc}, \cite{scalable_tensor_learning}, and \cite{spectral_regularization} use the regularizer $R(\W) = \sum_{k=1}^{K}\|W_k\|_*$, known as the overlapped trace norm, which promotes a lower Tucker (multilinear) rank in the recovered tensors.

Another popular trace norm regularizer is the latent trace norm regularizer. Methods that use this formulation, model the tensor as a sum of $K$ individual tensors, and the latent trace norm amounts to an $l_1$ norm regularizer that promotes sparsity. A few examples of such methods are \cite{ffw_lrtc}, which uses a Frank-Wolfe algorithm for optimization, and \cite{multitask_meets_tensor}, which uses a scaled variant of the latent trace norm.

The paper \cite{dual} uses a formulation that models the recovered tensor as a sum of non-sparse tensors and proposes a regularizer that uses $l_2$ norm regularizer as opposed to an $l_1$ norm. This allows for the development of a dual framework for tensor completion, which is solved using methods from Riemannian optimization.

Another class of tensor completion algorithms attempt to exploit the smoothness properties present in real-world tensor data like hyperspectral images and 3D images. The paper \cite{smooth_parafac} integrates the smooth PARAFAC decompositions for partially observed tensors and develops two variants using the total variation and quadratic variation. \cite{lrtc_tv} adopts total variation(TV) regularizer to formulate the model, and \cite{smf_lrtc} uses smooth matrix factorizations to incorporate tensor smoothness constraints.  

Tensor decomposition methods form another class of algorithms. Tensor decompositions like Tucker and CP decompositions act as generalizations of the familiar notion of singular value decomposition of matrices. \cite{geomCG} and \cite{RPrecon} exploit the Riemannian geometry of the set of fixed multi-linear rank tensors to efficiently learn the Tucker decomposition. \cite{bcpf} employs a Bayesian probabilistic CP decomposition model to recover the incomplete tensors. 

Other methods include \cite{t-svd}, which uses another form of tensor singular value decomposition to define a tensor rank known as the tubal rank. \cite{tmac} enforces the low-rank by factorizing the unfoldings of the tensor as low-rank matrices.

In \cite{ncpc}, an algorithm is proposed that uses a block coordinate descent method for nonnegative tensor completion, utilizing the CP decomposition. \cite{nonneg_ieee_access} performs nonnegative tensor completion based on low-rank Tucker decomposition.
Most of the research considering nonnegative tensors is devoted to learning nonnegative tensor decompositions. A few examples are \cite{nonneg_cp_hyper}, \cite{nonneg_tucker_decomp}, \cite{alternating_prox_nonneg}, \cite{nonneg_tensor_train}, etc. These methods cannot perform tensor completion task on incomplete data.

\section{Dual Framework}
Problem \eqref{eqn:primal_problem} models the nonnegative tensor completion using a regularizer that promotes low-rank solutions. We seek to learn $\W$ as the sum $\sum \W^{(k)}$ of $K$ tensors, as detailed in \cite{dual}. For our formulation, we use the regularizer 
\begin{equation*}
R(\W) = \sum_{k=1}^K \dfrac{1}{\lambda_k}\|W_k^{(k)}\|^2_*,
\end{equation*}

Following \cite{dual}, we develop a dual formulation for problem \eqref{eqn:primal_problem}, incorporating the structural constraint of nonnegativity into the formulation. We do this following a similar approach developed in \cite{structured_matrix_completion} for nonnegative matrix completion.

A key lemma \cite{lemma_proof} used in the development of the formulation is given below. 

\begin{lemma}\label{lemma:nuclear_norm}
For a matrix $X \in \R^{d \times T}$, the nuclear norm of $X$ satisfies the following relation:
\begin{equation*}
\|X\|_*^2 = \min_{\substack{\Theta \in \mathcal{P}^d\\ \text{range}(X) \subseteq \text{range}(\Theta)}} \innerproduct{\Theta^{\dagger}X}{X}
\end{equation*} where $\mathcal{P}^d = \{ S \in \R^{d \times d}:\, S \succeq 0, \text{tr}(S) = 1\} $, $\text{range}(\Theta) = \{\Theta z:\, z \in \R^d\}$, $\Theta^\dagger$ denotes the pseudo-inverse of $\Theta$. For a given $X$, the optimal $\Theta$ is $\bar{\Theta} = \sqrt{XX^T} / \text{tr}(\sqrt{XX^T}) $. 
\end{lemma}

Using the above lemma, we can write \eqref{eqn:primal_problem} as 
\begin{alignat}{3}\label{eqn:primal_rewritten}
\underset{\substack{\Theta_k \in \mathcal{P}^{n_k}, \W^{(k)}\\ k \in \{1,\cdots,K\}}}{\min}& &&C\Norm{\W_\Omega-\Y_\Omega}^2 + \sum_{k=1}^K \frac{1}{2\lambda_k} \innerproduct{\Theta_k^\dagger W^{(k)}_k}{W^{(k)}_k}  \nonumber\\
\text{subject to}& &&\hspace{1.5cm}\W \ge 0.
\end{alignat}

The following theorem provides the dual framework for the nonnegative low-rank tensor completion problem. It is a direct generalization of Theorem 1 in \cite{dual} to the case with nonnegative constraints.
\begin{theorem}
An equivalent partial dual formulation of the problem \eqref{eqn:primal_rewritten} is 
\begin{alignat}{3}\label{eqn:minimax_problem}
\underset{\substack{\Theta_k\in \mathcal{P}^{n_k},\\ k \in \{1,\ldots,K\}}}{\text{min}} &\;\underset{\substack{\Z\in\mathcal{C}\\ \s\in\Rptensor}}{\text{max}}\; \innerproduct{\Z}{\Y_\Omega} - \dfrac{1}{4C}\|\Z\|^2
\nonumber\\ 
&-\sum_{k=1}^K \frac{\lambda_k}{2} \innerproduct{(Z_k+S_k)}{\Theta_k(Z_k+S_k)},
\end{alignat} where $\mathcal{C} = \{\Z\in\Rtensor:\,\Z = \Z_\Omega\}$.
$\Z$ is the dual tensor variable corresponding to the primal problem \eqref{eqn:primal_rewritten}, $\s$ is the dual tensor variable corresponding to the nonnegative constraints.
\end{theorem}

\begin{proof}
Consider the inner problem of \eqref{eqn:primal_rewritten} over $\W^{(k)}$.
%
We introduce auxiliary variables $U_k$ with the associated constraints $U_k = W_k^{(k)}$. The Lagrangian of this problem will be
\begin{align}\label{eqn:lagrangian}
&\mathcal{L}(\W^{(1)},\ldots,\W^{(K)},U_1,\ldots,U_K, \Lambda_1,\ldots, \Lambda_K, \s) = \nonumber\\
&C\Norm{\bigg(\sum_{k=1}^K \W^{(k)}\bigg)_\Omega-\Y_\Omega}^2 + \sum_{k=1}^K \frac{1}{2\lambda_k} \innerproduct{\Theta_k^\dagger U_k}{ U_k} \nonumber\\
&+\sum_{k=1}^K \innerproduct{\Lambda_k}{W^{(k)}_k - U_k} - \innerproduct{\s}{\W}
\end{align} 
The dual function of the above will be given by 
\begin{alignat}{3}\label{eqn:dual_function}
\mathcal{Q}(\Theta_1,\ldots,\Theta_K,\Lambda_1,\ldots \Lambda_K,\s) = \underset{\substack{U_k, \W^{(k)}\\ k \in \{1,\ldots,K\}}}{\text{min}} \mathcal{L}
\end{alignat}
Applying the first-order KKT conditions, we get the following equations:
\begin{subequations}
\begin{align}
\fold{k}(\Lambda_k) &= \Z + \s,\label{eqn:lambda_Z}\\    
U_k &= \lambda_k\Theta_k \Lambda_k.\label{eqn:U_k}
\end{align}
\end{subequations}
where $\Z/(2C) = \Y_\Omega-\Big(\sum_{k=1}^K \W^{(k)}\Big)_\Omega$.
%
It can seen from the definition of $\Z$ that $\Z = \Z_\Omega$.

Using \eqref{eqn:lambda_Z} and \eqref{eqn:U_k}, we compute each term of \eqref{eqn:lagrangian} to be
\begin{equation*}
C\Norm{\bigg(\sum_{k=1}^K \W^{(k)}\bigg)_\Omega-\Y_\Omega}^2 = C\bigg(\dfrac{\|\Z\|^2}{4C^2}\bigg) = \dfrac{1}{4C}\|\Z\|^2,
\end{equation*}
\begin{align*}
\sum_{k=1}^K \frac{1}{2\lambda_k} &\innerproduct{\Theta_k^\dagger U_k}{U_k} - \innerproduct{\Lambda_k}{U_k} \\
&\hspace{1cm}= -\sum_{k=1}^K \frac{\lambda_k}{2} \innerproduct{(Z_k+S_k)}{\Theta_k(Z_k+S_k)},
\end{align*}
\begin{align*}
\sum_{k=1}^K \innerproduct{\Lambda_k}{W^{(k)}_k} - \innerproduct{\s}{\W} = 
\innerproduct{\Z}{\Y_\Omega} - \dfrac{1}{2C}\|\Z\|^2.
\end{align*}
Summing the terms, we obtain the expression for the dual function as
\begin{equation}
\mathcal{Q} = \innerproduct{\Z}{\Y_\Omega} - \dfrac{\|\Z\|^2}{4C} -\sum_{k=1}^K \frac{\lambda_k}{2} \innerproduct{(Z_k+S_k)}{\Theta_k(Z_k+S_k)}.\nonumber
\end{equation}
This gives the minimax problem \eqref{eqn:minimax_problem}. From \eqref{eqn:lambda_Z} and \eqref{eqn:U_k}, we can deduce the relation between optimal points of primal and minimax problems.
\end{proof}

If $\{\bar{\Theta}_1,\ldots,\bar{\Theta}_K,\bar{\Z},\bar{\s}\}$ is the optimal solution of \eqref{eqn:minimax_problem}, then the reconstructed tensor is given by $\bar{\W} = \sum_{k=1}^K \bar{\W}^{(k)}$ where $\bar{\W}^{(k)} = \lambda_k(\bar{\Z}+\bar{\s})\times_k\bar{\Theta}_k$ for all $k$. This factorization gives us a decoupling of the low-rank and nonnegative constraints enforced on $\W$ in \eqref{eqn:primal_rewritten} - the low-rank constraint is enforced by $\Theta_k$, the nonnegative constraints are encoded in $\s$, and $\Z$ corresponds to the dual variables of the primal problem.



\section{Proposed Algorithm}

Since $\Theta_k \in \mathcal{P}^{n_k}$ we can enforce the rank constraint explicitly by factorizing $\Theta_k$ as $\Theta_k = U_k U_k^T$, $U_k \in \mathcal{S}^{n_k}_{r_k}$, where $\mathcal{S}^n_r = \{U\in\R^{n\times r}:\|U\|_F = 1\}$. We rewrite \eqref{eqn:minimax_problem} as
\begin{equation}\label{eqn:outer_manifold_problem_nn}
\underset{U\in \mathcal{S}^{n_1}_{r_1}\times\cdots\times \mathcal{S}^{n_K}_{r_K}}{\min}\;g(U),
\end{equation}
where $U = (U_1,\ldots, U_K)$, and $g(U)$ is the optimal value of the problem
\begin{align}\label{eqn:inner_strong_convex_problem_nn}
g(U) = &\underset{\substack{\Z\in\mathcal{C}\\\s\in\Rptensor}}{\max} \innerproduct{\Z}{\Y_\Omega} - \dfrac{\|\Z\|^2}{4C} \nonumber\\
&-\sum_{k=1}^K \frac{\lambda_k}{2} \Norm{U_k^T(Z_k+S_k)}^2.
\end{align}

\subsection{Convex Optimization Problem}
The optimization problem \eqref{eqn:inner_strong_convex_problem_nn} is a convex optimization problem over the variables $\Z$ and $\s$, for a given $U$, hence it has a unique solution.


The problem \eqref{eqn:inner_strong_convex_problem_nn}  is solved separately for $\Z$ and $\s$ using an alternating minimization method. Equating the gradient of objective  with respect to $\Z$ to zero, we get
\begin{align}
\dfrac{\Z_\Omega}{2C} + \sum_{k=1}^K \lambda_k (\Z\times_k& U_kU_k^T)_\Omega =\Y_\Omega \nonumber\\ 
&- \sum_{k=1}^K \lambda_k (\s \times_k U_kU_k^T)_\Omega. \label{eqn:linear_system_Z}
\end{align}
This is a sparse linear system in $\Z$, which can be solved using linear conjugate gradient method. For various preconditioned CG approaches, see \cite{saad2003,benzi2002,das2020,das2021,katyan2020,mehta2020,kumar2014,kumar2016,kumar2013,kumar2013b,rampalli2018,kumar2010c,aggarwal2019,kumar2015b,kumar2010d,kumar2015e,niu2010,kumar2011c,kumar2014s}.

Problem \eqref{eqn:inner_strong_convex_problem_nn} has only one term involving $\s$. Hence, the optimization problem over $\s$ reduces to
\begin{align}
\underset{\s \in\Rptensor}{\min}  \sum_{k=1}^K \frac{\lambda_k}{2} \Norm{U_k^TZ_k + U_k^TS_k}^2. \label{eqn:NNLS_S}
\end{align}
 This is a nonnegative least squares (NNLS) problem. We use the method detailed in \cite{nnls_sra}, modified to suit our objective. 

\subsection{Riemannian Optimization Problem}

Given the optimizer $(\hat{\Z},\hat{\s})$ of \eqref{eqn:inner_strong_convex_problem_nn}, we compute $g$ at a point $U$ as 
\begin{align}
\label{eqn:direct_cost}
g(U) = \innerproduct{\hat{\Z}}{\Y_\Omega} - \dfrac{\|\hat{\Z}\|^2}{4C} -\sum_{k=1}^K \frac{\lambda_k}{2} \Norm{U_k^T(\hat{Z}_k+\hat{S}_k)}^2.
\end{align}
The set $\mathcal{S}^n_r$ is a Riemannian manifold, known as the spectrahedron manifold. The constraint set $\mathcal{S}^{n_1}_{r_1}\times\cdots\times \mathcal{S}^{n_K}_{r_K}$, therefore forms a product manifold and problem \eqref{eqn:outer_manifold_problem_nn} is an optimization problem on a manifold. 

To develop optimization algorithms on manifolds \cite{han2023,naram2022}, we need a few geometric tools. We delegate development of the specific tools to \cite{absil_spectrahedron} and \cite{dual}. For an introduction to optimization on general manifolds, we refer \cite{absil_book} and \cite{boumal_book}. 

For our case, the Euclidean gradient for $g$ can be computed as 
\begin{equation*}
    \nabla g(U) = - (\lambda_1 A_1, \ldots, \lambda_K A_K),
\end{equation*}
where $A_k = (\hat{Z_k} + \hat{S_k})(\hat{Z_k}+\hat{S_k})^T U_k$, for $1\le k\le K$. We use a generalization of non-linear conjugate gradient algorithm to Riemannian manifolds  \cite{riemannian_cg} to solve problem \eqref{eqn:outer_manifold_problem_nn}.

The proposed algorithm is detailed in Algorithm \ref{alg:rcg}. The reconstructed tensor is given by
\begin{equation*}
    \hat{\W} = \sum_{k=1}^{K} \lambda_k(\hat{\Z}+\hat{\s}) \times_k (U_k U_k^T).
\end{equation*}

\begin{algorithm}[H]
\footnotesize
	\centering
	\caption{Proposed Algorithm for Nonnegative Tensor Completion}\label{alg:rcg}
	\begin{algorithmic}[1]
		\Require $\Y_\Omega$, rank=$(r_1,\ldots,r_K)$, $\tau$, $(\lambda_1,\ldots,\lambda_K)$  \Comment{Input parameters}
		\For{$t = 1,2,\cdots$} 
		\State Check Termination: if $\|\nabla g(U^{(t)})\| \le \tau$ then break
        \State Compute $\hat{\Z}^{(t)}$ in \eqref{eqn:linear_system_Z} using conjugate gradient algorithm
		\State Compute $\hat{\s}^{(t)}$ in \eqref{eqn:NNLS_S} using NNLS solver
		\State Compute cost $g(U^{(t)})$ and gradient $\nabla g(U^{(t)})$
		\State Update $U$: $U^{(t+1)}$ = {\tt RiemannianCG-update}$(U^{(t)})$
		\EndFor
	    \State \textbf{Output:} $\hat{\W} = \sum_{k=1}^{K} \lambda_k(\hat{\Z}+\hat{\s}) \times_k (U_k U_k^T)$
	\end{algorithmic}
\end{algorithm}
\normalsize

\subsection{Complexity}
\begin{itemize}
    \item Step 3 (Computing $\hat{\Z}$): We use the linear conjugate gradient algorithm to solve the linear system \eqref{eqn:linear_system_Z}. The major cost in this step is to compute the matrix products $U_k^T Z_k$ and $U_k^T S_k$, for $k \in \{1,\ldots,K\}$. We can exploit the sparse structure of the problem to compute the products in $O(|\Omega|r_k)$ steps, and hence, if the linear solver takes $T_{cg}$ iterations, the total cost of this step is $O\left (\sum_{k=1}^K T_{cg}|\Omega|r_k \right)$.
    \item Step 4 (Computing $\hat{\s}$): For each iteration of the NNLS algorithm, we need to compute the cost function in \eqref{eqn:NNLS_S} and its gradient with respect to $\s$. Both of these operations can be computed in a similar manner as done for $\Z$, and the total cost of this step is $O\left(\sum_{k=1}^K T_{nnls}|\Omega|r_k\right)$, where $T_{nnls}$ is the number of iterations of NNLS algorithm.
    \item Step 5 (Computing cost and gradient): We can compute $g(U)$ from \eqref{eqn:direct_cost} given $\hat{\Z}$ and $\hat{\s}$ computed in previous steps. This can be done in $O(K|\Omega|)$. The gradient requires computing the matrix products $(Z_k+S_k)(Z_k+S_k)^T U_k$, and we can do this in $O(|\Omega|r_k)$. Hence, total cost for computing the gradient is $O(\sum_{k=1}^{K}|\Omega|r_k)$.
    \item Step 6 (Riemannian Conjugate Gradient): Search direction and step length are computed in this step. Then the current solution $U^{(t)}$ is updated to $U^{(t+1)}$ by performing retraction at the $U^{(t)}$ along the search direction. This step ensures that the update remains on the product manifold. These operations can be done in $O(\sum_{k=1}^K n_k r_k^2+\sum_{k=1}^K r_k^3)$.
\end{itemize}
Therefore, the overall per-iteration complexity of the proposed algorithm is 
\begin{equation*}
O\bigg((T_{cg}+T_{nnls})|\Omega|\sum_{k=1}^K{r_k}+\sum_{k=1}^K n_k r_k^2+\sum_{k=1}^K r_k^3\bigg).    
\end{equation*}

We store all the tensors in the sparse format, and perform operations accordingly. Hence, the overall space complexity of the proposed algorithm is 
\begin{equation*}
O\bigg(|\Omega|+\sum_{k=1}^K n_k r_k\bigg).    
\end{equation*}

\section{Numerical Experiments}

\subsection{Experimental setup}

We have performed experiments on several publicly available datasets (see Table \ref{tab:description_datasets}). We compare the performance of our algorithm to other state-of-the-art tensor completion algorithms. The baseline algorithms used for comparison are given below. Note that, with the exception of NCPC, all the other baseline algorithms do not enforce non-negativity in the completed tensors. 

\begin{table}
\centering
\begin{tabular}[H]{c|c|c}
\toprule
\multicolumn{1}{c}{Type} & \multicolumn{1}{c}{Dataset} & \multicolumn{1}{c}{Dimensions}\\
\midrule
	Hyperspectral & {\tt Ribeira} & $203\times 268\times 33$  \\
	Hyperspectral & {\tt Braga}   & $203\times 268\times 33$  \\
	Hyperspectral & {\tt Ruivaes} & $203\times 268\times 33$  \\
	Video         & {\tt Tomato}  & $320\times 242\times 167$ \\
	Video         &{\tt Container}& $144\times 176\times 150$ \\
	Video         & {\tt Hall}    & $144\times 176\times 150$ \\
	Video         & {\tt Highway} & $144\times 176\times 150$ \\
	Color Image   & {\tt Baboon}  & $256\times 256\times 3$   \\
	Color Image   & {\tt Splash}  & $512\times 512\times 3$   \\
\bottomrule
\end{tabular}
\caption{Description of datasets.}
\label{tab:description_datasets}
\end{table}

\begin{enumerate}
    \item {\tt Dual} \cite{dual}: A dual framework for low-rank tensor completion using a variant of the latent trace norm regularizer.
    \item {\tt RPrecon} \cite{RPrecon}: A low-rank tensor completion algorithm with a multi-linear rank constraint using Riemannian preconditioning.
    \item {\tt geomCG} \cite{geomCG}: An algorithm for tensor completion using optimization on the manifold of fixed multi-linear rank tensors.
    \item {\tt NCPC} \cite{ncpc}: A nonnegative tensor completion method using the CP decomposition.
    \item {\tt TMac} \cite{tmac}: An alternating minimization algorithm that uses parallel matrix factorization.
    \item {\tt LRTC-TV} \cite{lrtc_tv}: An ADMM based algorithm that uses total variation regularization to enforce smoothness.
    \item {\tt SMF-LRTC} \cite{smf_lrtc}: An algorithm that enforces smoothness constraint on factor matrices. 
    \item {\tt FFW} \cite{ffw_lrtc}: An algorithm with scaled latent nuclear norm using the Frank-Wolfe algorithm.
\end{enumerate}
We randomly sample 10\% of the tensor entries and use it as training data. The metric we use for evaluation is the {\tt RMSE} between the reconstructed and original tensors
\begin{equation*}
    {\tt RMSE} = \sqrt{\dfrac{\|\W-\W_{true}\|_F^2}{n_1n_2n_3}}.
\end{equation*}
The proposed method is implemented based on {\tt Dual} code. It uses {\tt MANOPT} library \cite{manopt} for implementing outer problem \eqref{eqn:outer_manifold_problem_nn} on manifolds. For the nonnegative least squares problem in \eqref{eqn:inner_strong_convex_problem_nn}, we use the {\tt NNLS} code \cite{nnls_sra} modified to work with our objective.

\begin{figure}[t]	
	\centering			
	\subfloat[RMSE v/s Iters]{\includegraphics[width=0.5\linewidth]{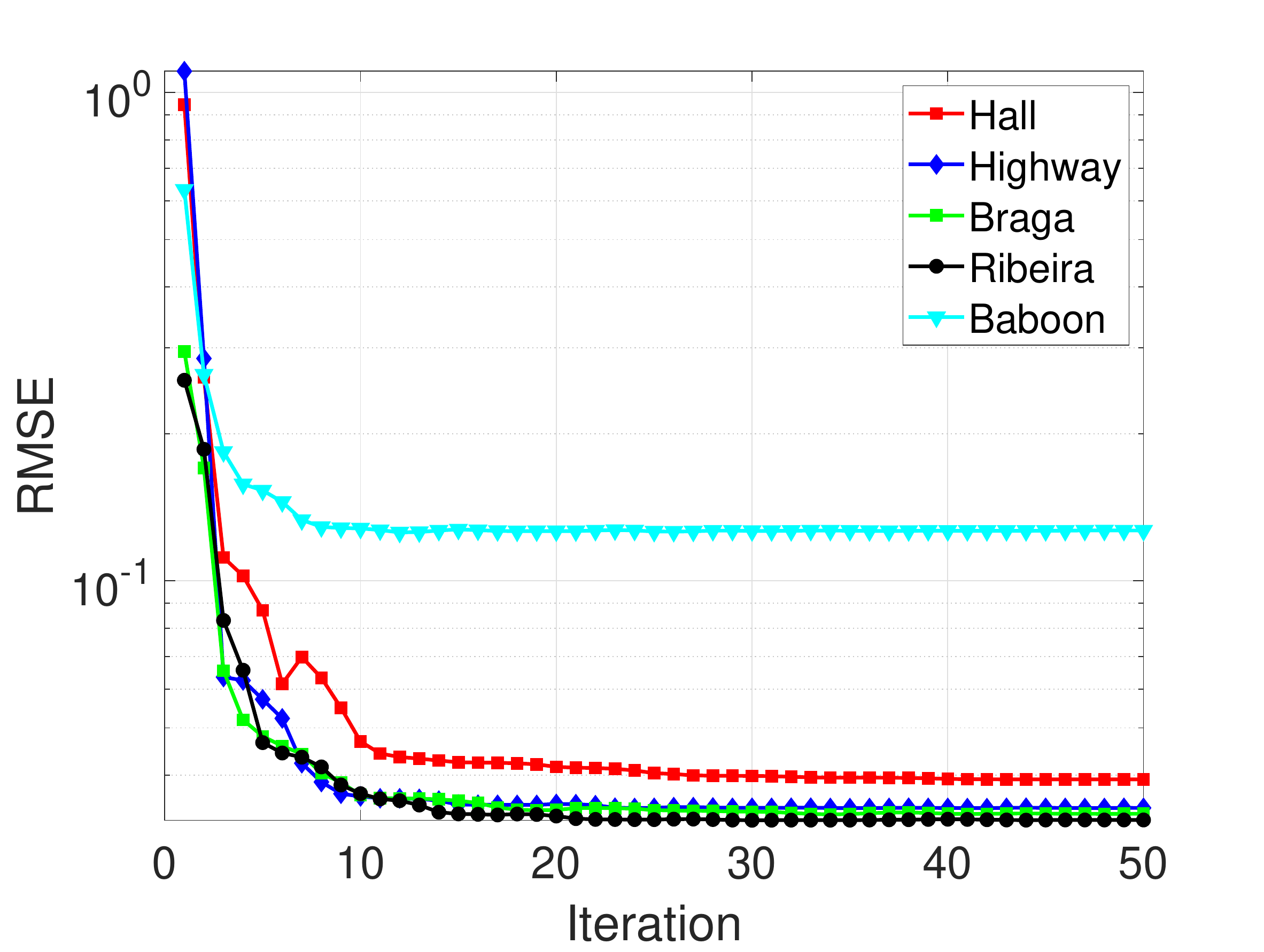}}%
	\hfil
	\subfloat[RMSE v/s Rank]{\includegraphics[width=0.5\linewidth]{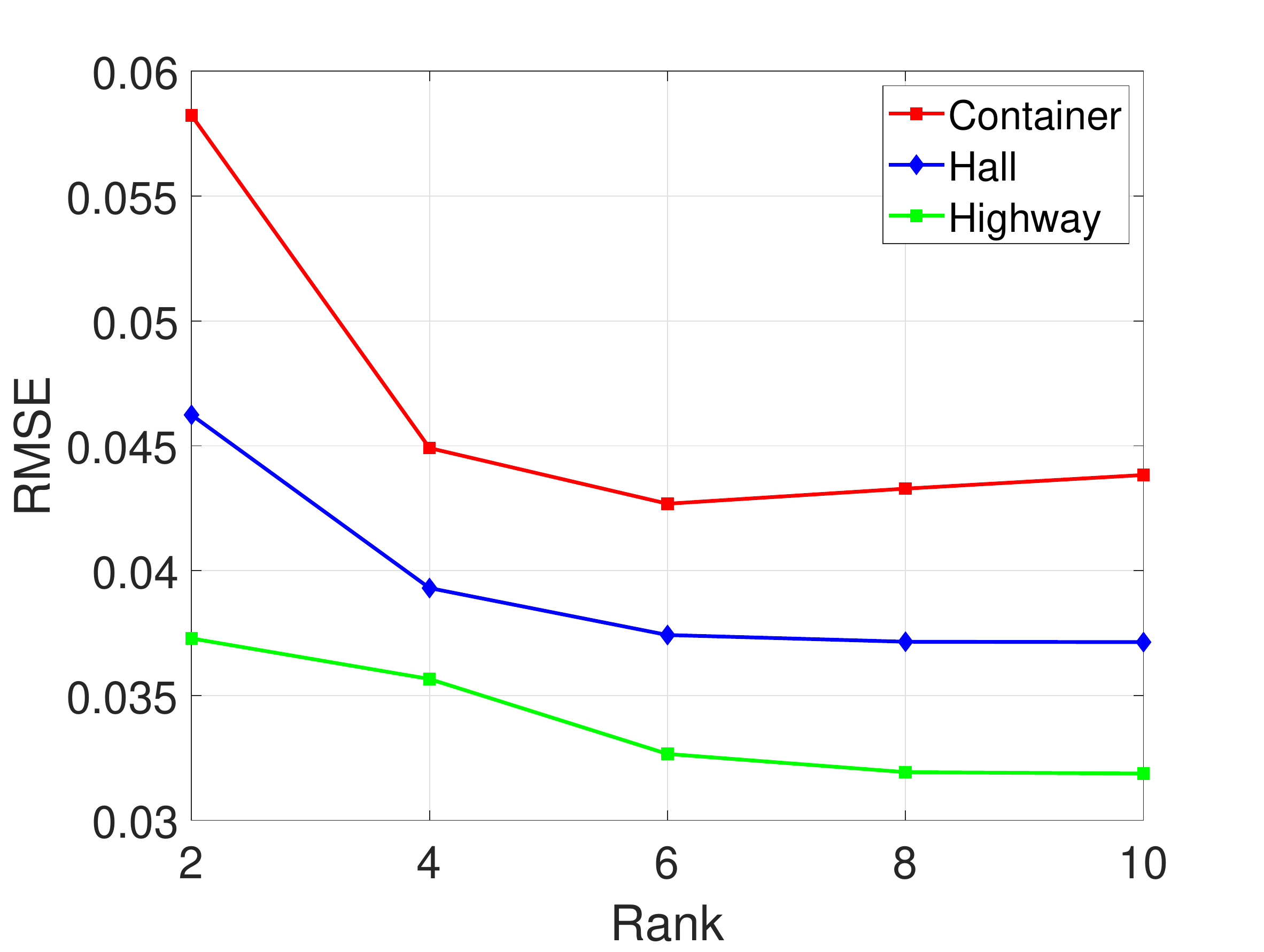}}%
	\hfil
	\caption{Variation of RMSE with iterations and rank. In the iterations plot, RMSE is in log scale. In the rank plot, the rank is taken to be value of X-label times $[1,1,1]$.}
    \label{fig:ablation_plots}
\end{figure}

\begin{figure}[t]	
	\centering			
	\subfloat{\includegraphics[width=0.18\linewidth]{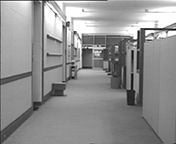}}%
	\hfil
	\subfloat{\includegraphics[width=0.18\linewidth]{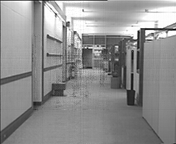}}%
	\hfil	
	\subfloat{\includegraphics[width=0.18\linewidth]{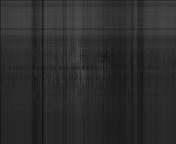}}%
	\hfil
	\subfloat{\includegraphics[width=0.18\linewidth]{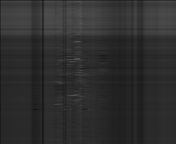}}%
	\hfil					
	\subfloat{\includegraphics[width=0.18\linewidth]{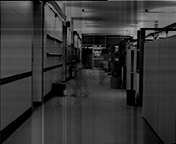}}%
	\hfil
    \setcounter{subfigure}{0}
	\subfloat[$\W_{real}$]{\includegraphics[width=0.18\linewidth]{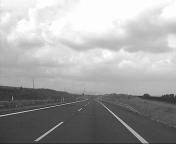}}%
	\hfil
	\subfloat[$\W$]{\includegraphics[width=0.18\linewidth]{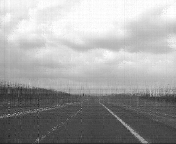}}%
	\hfil	
	\subfloat[$\W^{(1)}$]{\includegraphics[width=0.18\linewidth]{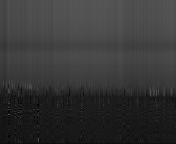}}%
	\hfil
	\subfloat[$\W^{(2)}$]{\includegraphics[width=0.18\linewidth]{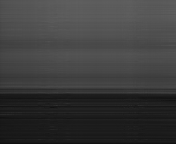}}%
	\hfil					
	\subfloat[$\W^{(3)}$]{\includegraphics[width=0.18\linewidth]{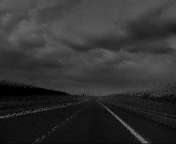}}%
	\hfil
	\caption{Original frame, reconstructed frame and components of reconstructed frame from Hall and Highway videos.}
	\label{fig:component_images}
\end{figure}

\begin{table*}[t]
\centering
\begin{tabular}{|c|c|c|c|c|c|c|c|c|c|c|c|}\hline
      Dataset  & Prop    & Dual   & RPrecon & geomCG  & NCPC    &  TMac   & LRTC-TV & SMF-LRTC&  FFW   \\ \hline 
      Ribeira  & \textbf{0.03090}  & \underline{0.03093} & 0.0454 & 0.06593 & 0.16465 & 0.1644  & 0.04984 & 0.04159 & 0.0696      \\ \hline
      Braga    & \textbf{0.02817}  & \underline{0.03054} & 0.0939 & 0.07348 & 0.07348 & 0.20226 & 0.20227 & 0.06445 & 0.06691     \\ \hline
      Ruivaes  & \textbf{0.029211} & 0.04969 & 0.0352  & 0.07146 & 0.14059 & 0.14073 & \underline{0.02955} & 0.040223 & 0.05437 \\ \hline
      Tomato   & \textbf{0.04282}  & \underline{0.04286} & 0.0589 & 0.05895 & 0.44175 & 0.44203 & 0.04638 & 0.053002& 0.11463     \\ \hline
      Container& \textbf{0.044908} & \underline{0.04693} & 0.0645 & 0.06452 & 0.5433  & 0.54413 & 0.09773 & 0.05894 & 0.15116     \\ \hline
      Hall     & \textbf{0.03696}  & \underline{0.03702} & 0.0687 & 0.06879 & 0.53629 & 0.53717 & 0.09409 & 0.06327 & 0.06327     \\ \hline
      Highway  & \textbf{0.03255}  & \underline{0.03652} & 0.0405 & 0.04055 & 0.6416  & 0.64261 & 0.04171 & 0.03723 & 0.11204     \\ \hline
      Baboon   & \underline{0.11943} & 0.33258 & 0.1563 & 2.7629  & 0.53315 & 0.5210 & \textbf{0.08729}   & 0.13456 & 0.14222     \\ \hline
      Splash   & \underline{0.06371} & 0.33475 & 0.3291 & 1.6897  & 0.50346 & 0.50014 & \textbf{0.05262}   & 0.09331 & 0.07779     \\ \hline
\end{tabular}
\caption{RMSE of various methods. The best result among all methods is in bold and second best are underlined.}
\label{tab:rmse}
\end{table*}

\begin{figure*}[t]	
	\centering			
	\subfloat{\includegraphics[width=0.095\linewidth]{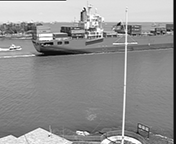}}%
	\hfil
	\subfloat{\includegraphics[width=0.095\linewidth]{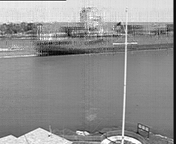}}%
	\hfil	
	\subfloat{\includegraphics[width=0.095\linewidth]{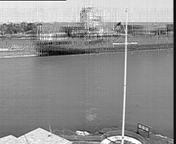}}%
	\hfil
	\subfloat{\includegraphics[width=0.095\linewidth]{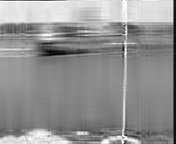}}%
	\hfil					
	\subfloat{\includegraphics[width=0.095\linewidth]{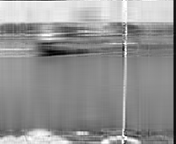}}%
	\hfil
	\subfloat{\includegraphics[width=0.095\linewidth]{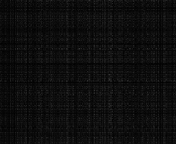}}%
	\hfil
	\subfloat{\includegraphics[width=0.095\linewidth]{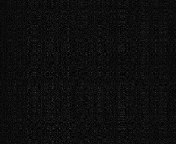}}%
	\hfil
	\subfloat{\includegraphics[width=0.095\linewidth]{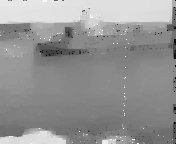}}%
	\hfil					
	\subfloat{\includegraphics[width=0.095\linewidth]{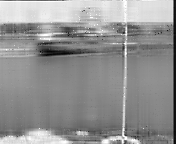}}%
	\hfil
	\subfloat{\includegraphics[width=0.095\linewidth]{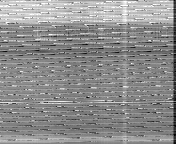}}%
	\hfil
	\setcounter{subfigure}{0}
	\subfloat{\includegraphics[width=0.095\linewidth]{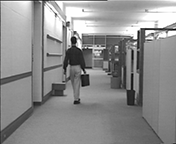}}%
	\hfil
	\subfloat{\includegraphics[width=0.095\linewidth]{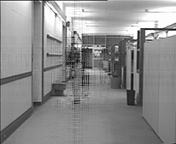}}%
	\hfil	
	\subfloat{\includegraphics[width=0.095\linewidth]{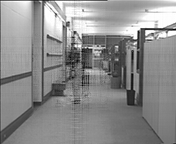}}%
	\hfil
	\subfloat{\includegraphics[width=0.095\linewidth]{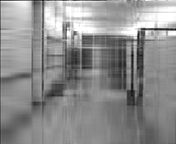}}%
	\hfil					
	\subfloat{\includegraphics[width=0.095\linewidth]{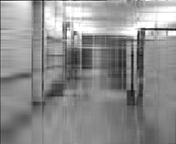}}%
	\hfil
	\subfloat{\includegraphics[width=0.095\linewidth]{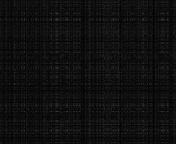}}%
	\hfil
	\subfloat{\includegraphics[width=0.095\linewidth]{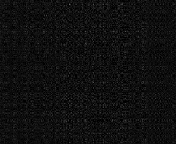}}%
	\hfil
	\subfloat{\includegraphics[width=0.095\linewidth]{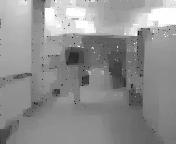}}%
	\hfil					
	\subfloat{\includegraphics[width=0.095\linewidth]{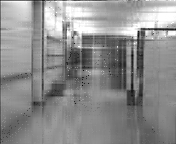}}%
	\hfil
	\subfloat{\includegraphics[width=0.095\linewidth]{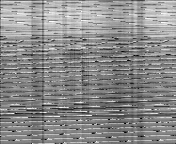}}%
	\hfil
	\setcounter{subfigure}{0}
	\subfloat{\includegraphics[width=0.095\linewidth]{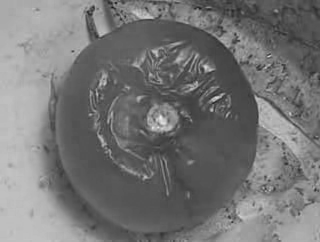}}%
	\hfil
	\subfloat{\includegraphics[width=0.095\linewidth]{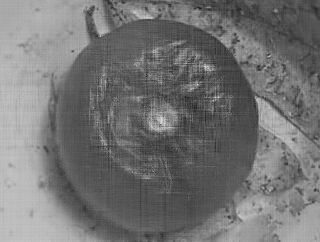}}%
	\hfil	
	\subfloat{\includegraphics[width=0.095\linewidth]{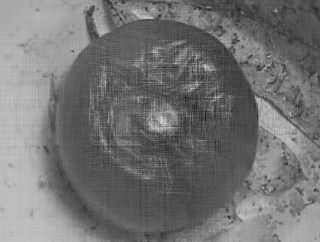}}%
	\hfil
	\subfloat{\includegraphics[width=0.095\linewidth]{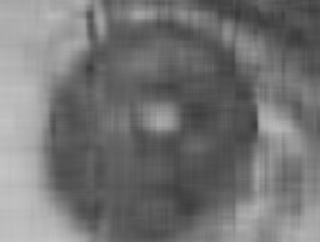}}%
	\hfil					
	\subfloat{\includegraphics[width=0.095\linewidth]{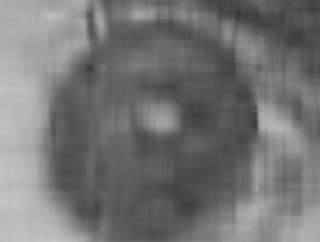}}%
	\hfil
	\subfloat{\includegraphics[width=0.095\linewidth]{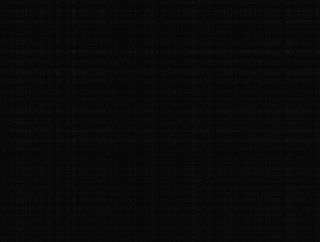}}%
	\hfil
	\subfloat{\includegraphics[width=0.095\linewidth]{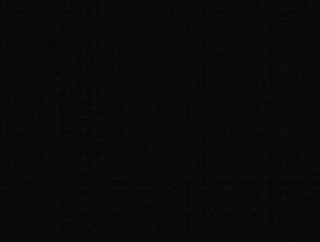}}%
	\hfil
	\subfloat{\includegraphics[width=0.095\linewidth]{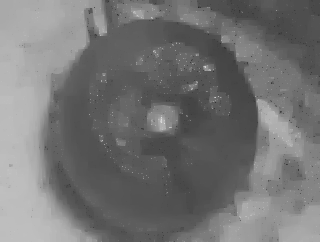}}%
	\hfil					
	\subfloat{\includegraphics[width=0.095\linewidth]{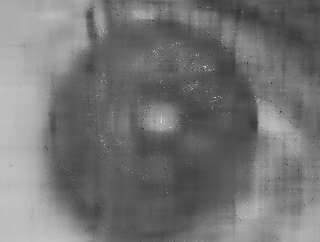}}%
	\hfil
	\subfloat{\includegraphics[width=0.095\linewidth]{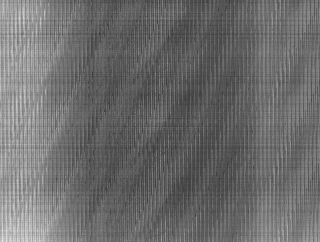}}%
	\hfil
	\setcounter{subfigure}{0}
	\subfloat{\includegraphics[width=0.095\linewidth]{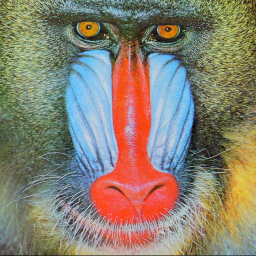}}%
	\hfil
	\subfloat{\includegraphics[width=0.095\linewidth]{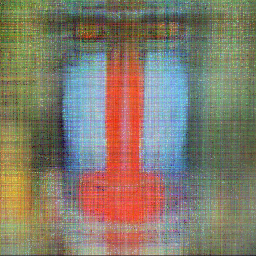}}%
	\hfil	
	\subfloat{\includegraphics[width=0.095\linewidth]{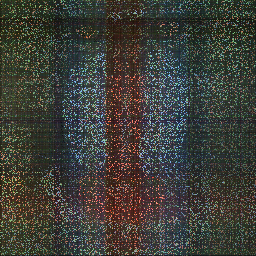}}%
	\hfil
	\subfloat{\includegraphics[width=0.095\linewidth]{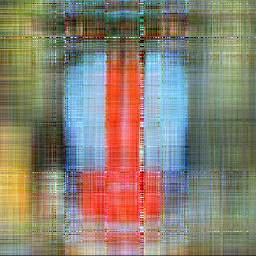}}%
	\hfil					
	\subfloat{\includegraphics[width=0.095\linewidth]{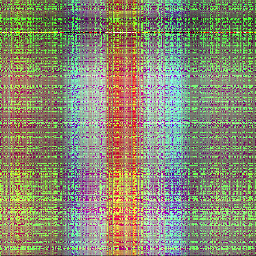}}%
	\hfil
	\subfloat{\includegraphics[width=0.095\linewidth]{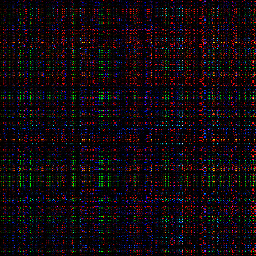}}%
	\hfil
	\subfloat{\includegraphics[width=0.095\linewidth]{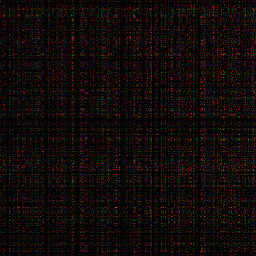}}%
	\hfil
	\subfloat{\includegraphics[width=0.095\linewidth]{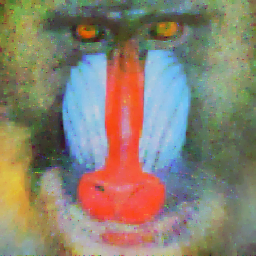}}%
	\hfil					
	\subfloat{\includegraphics[width=0.095\linewidth]{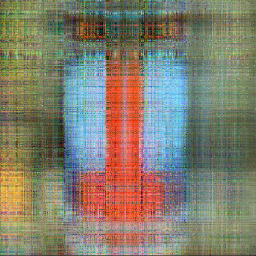}}%
	\hfil
	\subfloat{\includegraphics[width=0.095\linewidth]{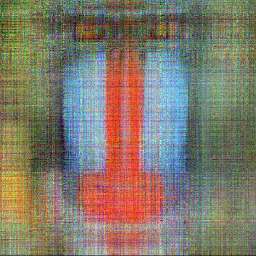}}%
	\hfil
	\setcounter{subfigure}{0}
	\subfloat{\includegraphics[width=0.095\linewidth]{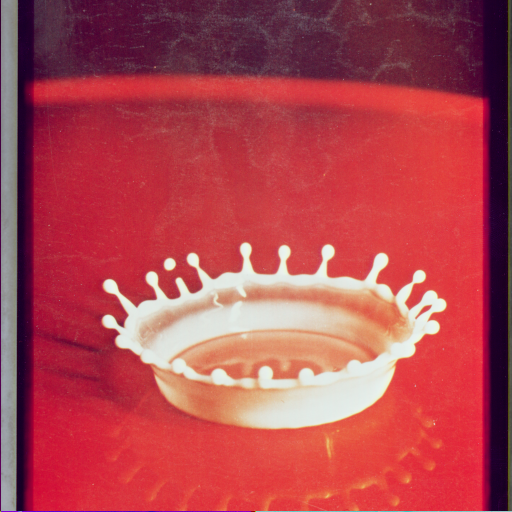}}%
	\hfil
	\subfloat{\includegraphics[width=0.095\linewidth]{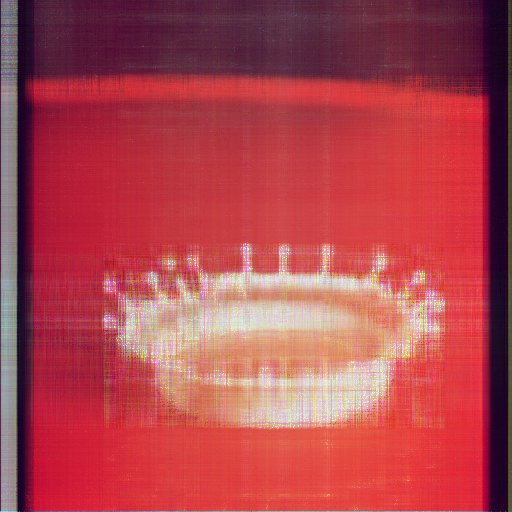}}%
	\hfil	
	\subfloat{\includegraphics[width=0.095\linewidth]{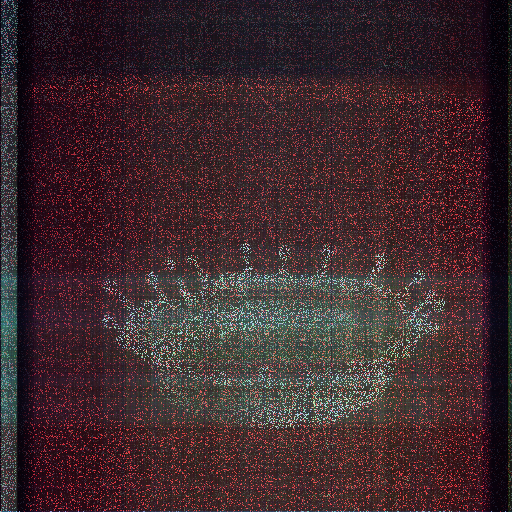}}%
	\hfil
	\subfloat{\includegraphics[width=0.095\linewidth]{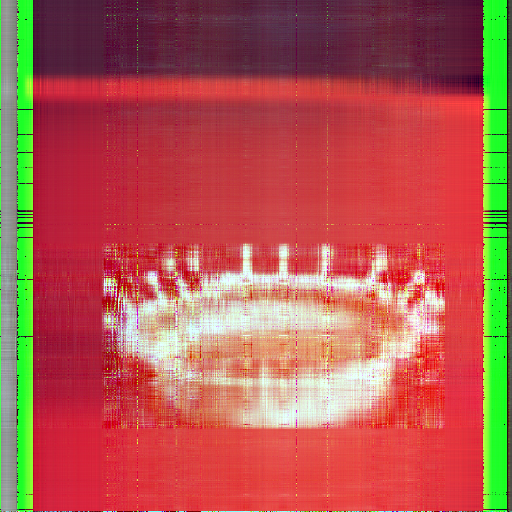}}%
	\hfil					
	\subfloat{\includegraphics[width=0.095\linewidth]{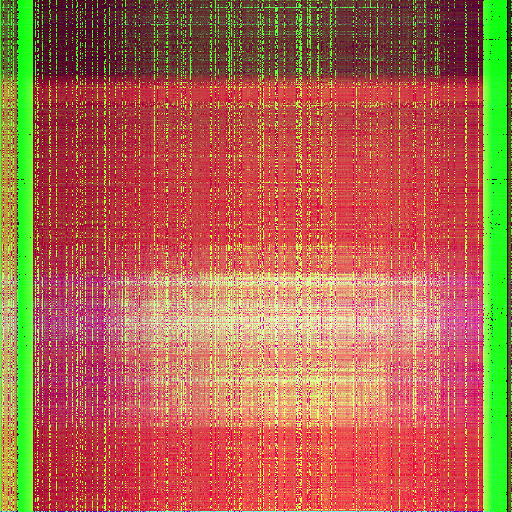}}%
	\hfil
	\subfloat{\includegraphics[width=0.095\linewidth]{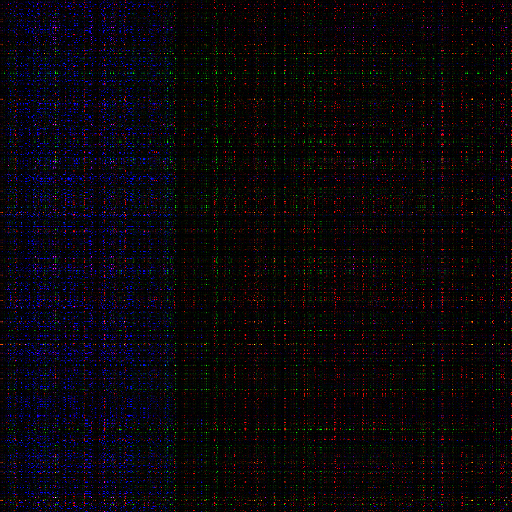}}%
	\hfil
	\subfloat{\includegraphics[width=0.095\linewidth]{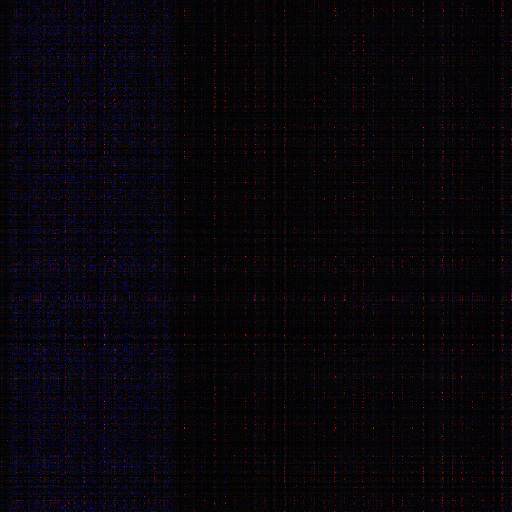}}%
	\hfil
	\subfloat{\includegraphics[width=0.095\linewidth]{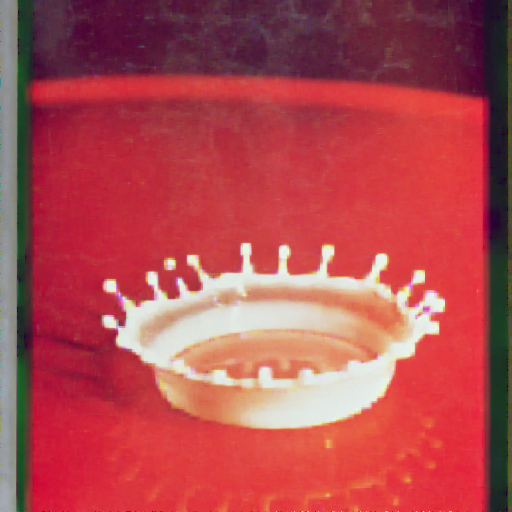}}%
	\hfil					
	\subfloat{\includegraphics[width=0.095\linewidth]{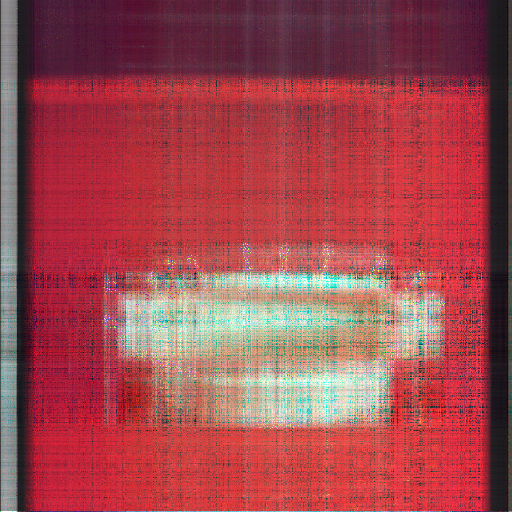}}%
	\hfil
	\subfloat{\includegraphics[width=0.095\linewidth]{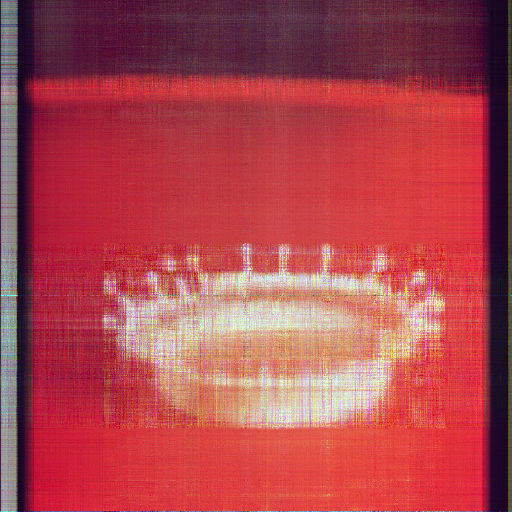}}%
	\hfil
	\setcounter{subfigure}{0}
	\subfloat{\includegraphics[width=0.095\linewidth]{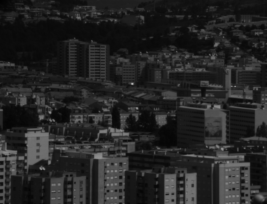}}%
	\hfil
	\subfloat{\includegraphics[width=0.095\linewidth]{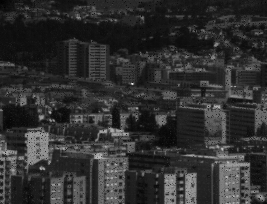}}%
	\hfil	
	\subfloat{\includegraphics[width=0.095\linewidth]{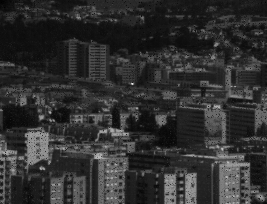}}%
	\hfil
	\subfloat{\includegraphics[width=0.095\linewidth]{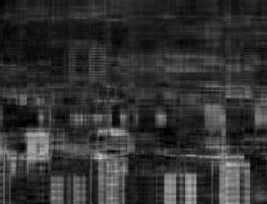}}%
	\hfil					
	\subfloat{\includegraphics[width=0.095\linewidth]{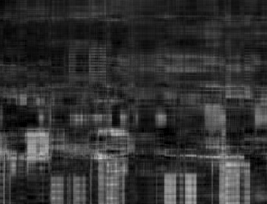}}%
	\hfil
	\subfloat{\includegraphics[width=0.095\linewidth]{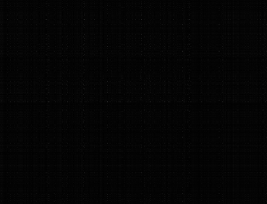}}%
	\hfil
	\subfloat{\includegraphics[width=0.095\linewidth]{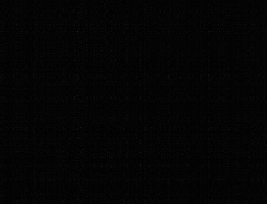}}%
	\hfil
	\subfloat{\includegraphics[width=0.095\linewidth]{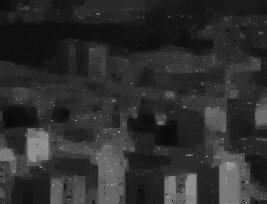}}%
	\hfil					
	\subfloat{\includegraphics[width=0.095\linewidth]{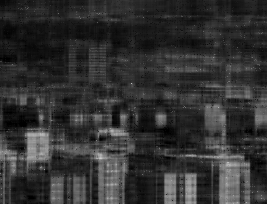}}%
	\hfil
	\subfloat{\includegraphics[width=0.095\linewidth]{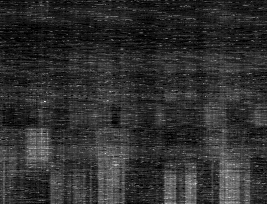}}%
	\hfil
	\setcounter{subfigure}{0}
	\subfloat[Original]{\includegraphics[width=0.095\linewidth]{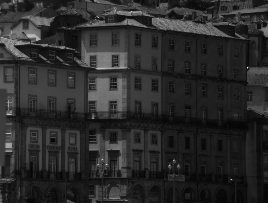}}%
	\hfil	
	\subfloat[Prop]{\includegraphics[width=0.095\linewidth]{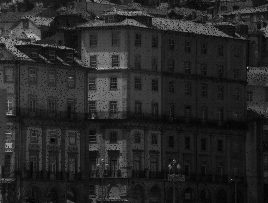}}%
	\hfil	
	\subfloat[Dual]{\includegraphics[width=0.095\linewidth]{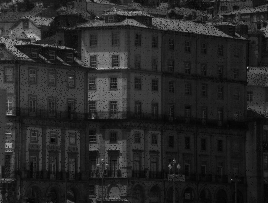}}%
	\hfil
	\subfloat[RPrec]{\includegraphics[width=0.095\linewidth]{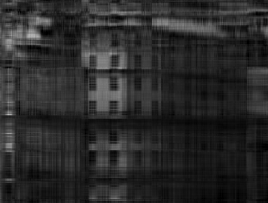}}%
	\hfil					
	\subfloat[geomCG]{\includegraphics[width=0.095\linewidth]{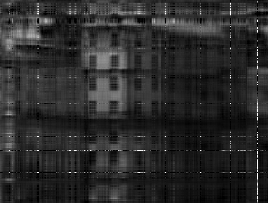}}%
	\hfil
	\subfloat[NCPC]{\includegraphics[width=0.095\linewidth]{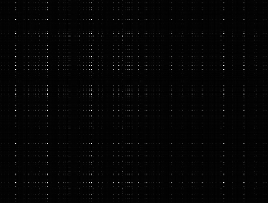}}%
	\hfil
	\subfloat[TMac]{\includegraphics[width=0.095\linewidth]{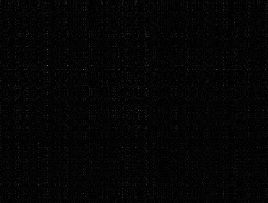}}%
	\hfil
	\subfloat[T-LRTC]{\includegraphics[width=0.095\linewidth]{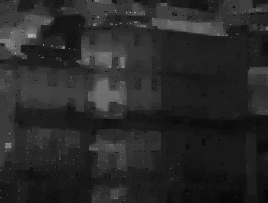}}%
	\hfil					
	\subfloat[S-LRTC]{\includegraphics[width=0.095\linewidth]{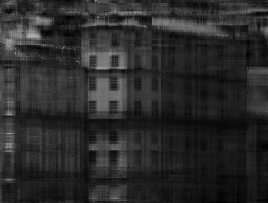}}%
	\hfil
	\subfloat[FFW]{\includegraphics[width=0.095\linewidth]{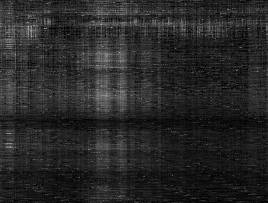}}%
	\caption{Original and Reconstructed Images for Different Algorithms given 10\% of the fraction as training data. The datasets shown from top to bottom are Container, Hall, Tomato, Baboon, Splash, Braga, Ribeira respectively. For videos, a random frame was chosen. From left to right: Original, Proposed, Dual \cite{dual}, RPrecon \cite{RPrecon}, geomCG \cite{geomCG}, NCPC \cite{ncpc}, TMac \cite{tmac}, LRTC-TV(T-LRTC) \cite{lrtc_tv}, SMF-LRTC(S-LRTC) \cite{smf_lrtc} and FFW-LRTC \cite{ffw_lrtc}.}
	\label{fig:reconst_images}
\end{figure*}

\subsection{Hyperparameters} 
The ranks $(r_1,r_2,r_3)$ are chosen as $(10,10,5)$ for all datasets, except color images where we chose $(10,10,3)$ since the dimension in mode-3 is less than $5$. We chose the regularization constants $\lambda_k$'s according to \cite{dual}. The maximum iterations for outer optimization problem \eqref{eqn:outer_manifold_problem_nn} was set to $200$ as no improvement in {\tt RMSE} is seen after $100$ iterations on most of the datasets. For the baseline algorithms using rank as a hyperparameter, we have chosen the same rank as in our case since it is sufficient for a variety of datasets (see \cite{dual}). Additional hyperparameters for each baseline were set as indicated in the code provided by the authors.

We consider the effect of {\tt RMSE} on the variation in hyperparameters. Fig. \ref{fig:ablation_plots}(a) shows the variation of {\tt RMSE} over iterations of the proposed algorithm. We see that the {\tt RMSE} decreases as the algorithm proceeds, and the decrease is rapid in the initial iterations. The {\tt RMSE} decreases monotonically, so we have chosen $200$ iterations as the threshold to guarantee good solutions. Fig. \ref{fig:ablation_plots}(b) shows the variation of {\tt RMSE} with the rank. Increasing rank decreases the {\tt RMSE}, but it quickly saturates to $10$, which can be due to the inherent rank of the dataset. This justifies our choice of hyperparameters.

\subsection{Image Completion}

The task of image completion is to reconstruct the original image tensor, given only partial observations. As mentioned earlier, we have randomly sampled $10\%$ of observations for training. We have experimented with several hyperspectral images (see Table \ref{tab:description_datasets}, \cite{hyperspectral_images}) where each data tensor contains stack of images measured at different wavelengths. Following \cite{dual} we resized these datasets to $203\times 268\times 33$ using bilinear interpolation. We have also considered two color images (see Table \ref{tab:description_datasets}, \cite{color_images}) which are naturally represented as third-order tensors.

We report the {\tt RMSE} in Table \ref{tab:rmse} and some of the reconstructed images in Fig.\;\ref{fig:reconst_images}. Our proposed algorithm outperforms the baseline algorithms in all hyperspectral datasets considered. The reconstructed images are of good quality, given only $10\%$ of data for training. In color image datasets, {\tt LRTC-TV} performs best. We expect that this is because the original images have the local smoothness property, which is exploited by {\tt LRTC-TV} through the smoothness constraints it enforces. However, the low-rank and nonnegative structure does not preserve such smoothness, which explains the performance of other algorithms. Nevertheless, the proposed algorithm achieves the best {\tt RMSE} next to {\tt LRTC-TV}. We believe this indicates the usefulness of nonnegative constraints. 

For hyperspectral images, {\tt LRTC-TV} performs badly. On {\tt Braga}, {\tt RMSE} of {\tt LRTC-TV} is 10 times that of the proposed algorithm. By comparing the reconstructed images, it can be seen that the smooth image produced by {\tt LRTC-TV} is an imperfect reconstruction, suggesting the lack of local smoothness property in this dataset. 

The effect of nonnegativity is more pronounced in color images, where the proposed algorithm achieves 3 times lower {\tt RMSE} on {\tt Baboon} and 5 times lower {\tt RMSE} on {\tt Splash} compared to {\tt Dual}. We see this effect in the reconstructed image of {\tt Baboon} and {\tt Splash}, where, perhaps due to negative entries, the reconstructed images appear darker.

\subsection{Video Completion}

Video completion task is the reconstruction of the frames of the video from the partial observations given. We considered several gray-scale videos (see Table \ref{tab:description_datasets}, \cite{video_sequences}) which form third-order tensors.

Fig. \ref{fig:component_images} shows the component frames, $\W^{(k)}$'s, of the reconstructed frames, $\W$, of the proposed algorithm. For the video data, most of the information varies along the frames (i.e., along mode-$3$ rather than the other modes). Consequently, we see the frames of $\W^{(1)}$ and $\W^{(2)}$ to have less information, whereas the frame of $\W^{(3)}$ is close to the original frame. As we enforce the low-rank constraint on mode-$k$ of $\W^{(k)}$, each component has a compact representation that captures the original scene very well.

The proposed algorithm achieves the least {\tt RMSE} compared to the baselines (see Table \ref{tab:rmse}). In {\tt Hall}, the {\tt RMSE} scores of all baseline algorithms, except {\tt Dual}, is at least two times that of the proposed algorithm. Despite the increase in dimensions of the tensor compared to hyperspectral images, choosing the same rank $(10,10,5)$ gives the best {\tt RMSE} scores. The reconstructed image shown in Fig. \ref{fig:reconst_images} are significantly clear as indicated by the {\tt RMSE} scores. As mentioned earlier, in {\tt Tomato} and {\tt Hall}, we believe that the lack of local smoothness property leads to the failure of {\tt LRTC-TV} algorithm.

\section{Conclusion}

We have proposed a novel factorization for nonnegative low-rank tensor completion, $\W = \sum_{k=1}^K (\Z+\s)\times_k U_kU_k^T$. The factorization decouples the nonnegative constraint and low-rank constraint on $\s$ and $U_kU_k^T$ respectively. The resultant problem has a geometric structure in the constraints. We exploit this structure to propose a Riemannian optimization algorithm to solve the problem. On several real-world datasets, our proposed algorithm outperforms the state-of-the-art tensor completion algorithms. 

\section*{Acknowledgement}

Tanmay Kumar Sinha was supported by IIIT seed grant. Jayadev Naram thanks IHub-Data, IIIT Hyderabad for a research fellowship.

{\small
\bibliographystyle{ieee_fullname}

\begin{thebibliography}{8}


\bibitem{boumal_book}
Boumal, N.: An introduction to optimization on smooth manifolds, (2020). 
Accessed online from \url{http://www.nicolasboumal.net/book}

\bibitem{kolda_review}
Kolda, T. G. and Bader, B. W.: Tensor decompositions and applications, SIAM Review 51 (2009), no. 3, 455–500.


\bibitem{manopt}
Boumal, N. and Mishra, B. and Absil, P.-A. and Sepulchre, R.: {M}anopt, a {M}atlab Toolbox for Optimization on Manifolds. 
In: Journal of Machine Learning Research, vol. 15, pp. 1455--1459 (2014).
\url{https://www.manopt.org}

\bibitem{absil_book} 
Absil, P.-A., Mahony, R. and Sepulchre, R.: Optimization Algorithms on Matrix Manifolds. Princeton University Press, Princeton, NJ (2008).
Accessed online from \url{https://press.princeton.edu/absil}

\bibitem{absil_spectrahedron}
Journée, M., Bach, F., Absil, P.A. and Sepulchre, R., 2010. Low-rank optimization on the cone of positive semidefinite matrices. SIAM Journal on Optimization, 20(5), pp.2327-2351.

\bibitem{dual}
Nimishakavi, M., Jawanpuria, P. and Mishra, B.: A Dual Framework for Trace Norm Regularized Low-rank Tensor Completion, Conference on Neural Information Processing Systems (NeurIPS), 2018. \url{https://github.com/madhavcsa/Low-Rank-Tensor-Completion}

\bibitem{lemma_proof}
Argyriou, A., Evgeniou, T. and Pontil, M.: Multi-task feature learning, NIPS, 2006.

\bibitem{oldest_lrtc}
Liu, J., Musialski, P., Wonka, P., and Ye, J.: Tensor completion for estimating missing values in visual data. IEEE transactions on pattern analysis and machine intelligence, 35(1), 208-220, 2012.

\bibitem{geomCG}
Kressner, D., Steinlechner, M. and Vandereycken, B.: Low-rank tensor completion by Riemannian optimization, BIT Numerical Mathematics 54, no. 2 (2014): 447-468. \url{https://www.epfl.ch/labs/anchp/index-html/software/geomcg}

\bibitem{RPrecon}
Kasai, H. and Mishra, B., 2016, June. Low-rank tensor completion: a Riemannian manifold preconditioning approach. In International conference on machine learning (pp. 1012-1021). PMLR. \url{https://bamdevmishra.in/codes/tensorcompletion}

\bibitem{ncpc}
Xu, Y. and Yin, W., 2013. A block coordinate descent method for regularized multiconvex optimization with applications to nonnegative tensor factorization and completion. SIAM Journal on imaging sciences, 6(3), pp.1758-1789. \url{https://xu-yangyang.github.io/BCD}

\bibitem{spectral_regularization}
Signoretto, M., Dinh, Q. T., De Lathauwer, L., and Suykens, J. A.: Learning with tensors: a framework based on convex optimization and spectral regularization. Machine Learning, 94(3), 303-351, 2014.

\bibitem{scalable_tensor_learning}
Cheng, H., Yu, Y., Zhang, X., Xing, E., and Schuurmans, D.: Scalable and sound low-rank tensor learning. In Artificial Intelligence and Statistics (pp. 1114-1123). PMLR, 2016.

\bibitem{multitask_meets_tensor}
Wimalawarne, K., Sugiyama, M. and Tomioka, R.: Multitask learning meets tensor factorization: task imputation via convex optimization. Advances in neural information processing systems, 27, 2014, pp.2825-2833.

\bibitem{structured_matrix_completion}
Jawanpuria, P. and Mishra, B.: A unified framework for structured low-rank matrix learning, ICML, 2018.

\bibitem{ffw_lrtc}
Guo, X., Yao, Q. and Kwok, J. T.: Efficient sparse low-rank tensor completion using the frank-wolfe algorithm, AAAI, 2017. \url{https://github.com/quanmingyao/FFWTensor}

\bibitem{lrtc_tv}
Li, X., Ye, Y. and Xu, X.: Low-Rank Tensor Completion with Total Variation for Visual Data Inpainting. AAAI, 2017.

\bibitem{smooth_parafac}
Yokota, T., Zhao, Q. and Cichocki, A., 2016. Smooth PARAFAC decomposition for tensor completion. IEEE Transactions on Signal Processing, 64(20), pp.5423-5436.

\bibitem{smf_lrtc}
Zheng, Y., Huang, T., Ji, T., Zhao, X., Jiang, T., Ma, T.: Low-rank tensor completion via smooth matrix factorization, Applied Mathematical Modelling, Volume 70, 2019, Pages 677-695. \url{https://github.com/uestctensorgroup/code_SMFLRTC}

\bibitem{tmac}
Xu, Y., Hao, R., Yin, W. and Su, Z., 2013. Parallel matrix factorization for low-rank tensor completion. arXiv preprint arXiv:1312.1254. \url{https://xu-yangyang.github.io/TMac}

\bibitem{nnls_sra}
Kim, D., Sra, S. and Dhillon, I. S.: A non-monotonic method for large-scale non-negative least squares. Optimization Methods and Software, 28(5):1012–1039, 2013. \url{http://optml.mit.edu/work/soft/nnls.html}

\bibitem{t-svd}
Zhang, Z. and Aeron, S., 2016. Exact tensor completion using t-SVD. IEEE Transactions on Signal Processing, 65(6), pp.1511-1526.

\bibitem{nonneg_ieee_access}
Chen, B., Sun, T., Zhou, Z., Zeng, Y. and Cao, L., 2019. Nonnegative tensor completion via low-rank Tucker decomposition: model and algorithm. IEEE Access, 7, pp.95903-95914.

\bibitem{nonneg_cp_hyper}
Veganzones, M. A., Cohen, J. E., Farias, R. C., Chanussot, J. and Comon, P.: Nonnegative tensor CP decomposition of hyperspectral data, IEEE Trans. Geosci. Remote Sens., vol. 54, no. 5, pp. 2577–2588, May 2016.

\bibitem{nonneg_tucker_decomp}
Li, X., Ng, M. K., Cong, G., Ye, Y. and Wu, Q.: MR-NTD: Manifold regularization nonnegative Tucker decomposition for tensor data dimension reduction and representation, IEEE Transactions on Neural Networks and Learning Systems, vol. 28, no. 8, pp. 1787–1800, Aug. 2017.

\bibitem{nonneg_tensor_train}
Lee, N., Phan, A., Cong, F. and Cichocki, A.: Nonnegative tensor train decompositions for multi-domain feature extraction and clustering, NIPS, 2016, pp. 87–95.

\bibitem{alternating_prox_nonneg}
Xu, Y.: Alternating proximal gradient method for sparse nonnegative Tucker decomposition, Mathematical Programming Computation, vol. 7, no. 1, pp. 39–70, 2015.

\bibitem{bcpf}
Zhao, Q., Zhang, L. and Cichocki, A.: Bayesian CP factorization of incomplete tensors with automatic rank determination, IEEE Transactions on Pattern Analysis and Machine Intelligence 37 (2015), no. 9, 1751–1763.

\bibitem{riemannian_cg}
Boumal, N. and Absil, P.-A.: Low-rank matrix completion via preconditioned optimization on the {G}rassmann manifold. 
In: Linear Algebra and its Applications, vol. 475, pp. 200--239 (2015).

\bibitem{hyperspectral_images}
Foster, D. H., Amano, K., Nascimento, S. M. C., and Foster, M. J.: Frequency of metamerism in natural scenes. Journal of the Optical Society of America A, 23, 2359-2372, 2006.
 
\bibitem{color_images}
SIPI Image Database. \url{http://sipi.usc.edu/database/database.php?volume=misc}

\bibitem{video_sequences}
YUV Video Sequences. \url{http://trace.eas.asu.edu/yuv/index.html}

\bibitem{benzi2002}
M. Benzi. Preconditioning techniques for large linear systems: A
survey. Journal of Computational Physics, 182(2):418–477, 2002.

\bibitem{das2020}
S. Das, S. Katyan, and P. Kumar. Domain decomposition based preconditioned solver for bundle adjustment. In
R. Venkatesh Babu, Mahadeva Prasanna, and Vinay P. Namboodiri,
editors, Computer Vision, Pattern Recognition, Image Processing, and
Graphics, pages 64–75, Singapore, 2020. Springer Singapore.

\bibitem{das2021}
S. Das, S. Katyan, and P. Kumar. A deflation based
fast and robust preconditioner for bundle adjustment. In Proceedings of
the IEEE/CVF Winter Conference on Applications of Computer Vision
(WACV), pages 1782–1789, January 2021.

\bibitem{katyan2020}
S. Katyan, S. Das, and P. Kumar. Two-grid precon-
ditioned solver for bundle adjustment. In 2020 IEEE Winter Conference
on Applications of Computer Vision (WACV), pages 3588–3595, 2020

\bibitem{mehta2020}
K. Mehta, A. Mahajan, P. Kumar Effects of spectral normal-
ization in multi-agent reinforcement learning. In IJCNN, 2023

\bibitem{kumar2014}
P. Kumar. Aggregation based on graph matching and inexact coarse
grid solve for algebraic two grid. International Journal of Computer
Mathematics, 91(5):1061–1081, 2014

\bibitem{kumar2009}
P. Kumar, L. Grigori, F. Nataf, Q. Niu, Combinative preconditioning based on relaxed nested factorization and tangential filtering preconditioner, INRIA, HAL Id: inria-00392881, 2009. 

\bibitem{kumar2010b}
L. Grigori, P. Kumar, F. Nataf, K. Wang, A class of multilevel parallel preconditioning strategies, INRIA 7410, 2010. 

\bibitem{rampalli2018}
S. Rampalli, N. Sehgal, I. Bindlish, T. Tyagi, P. Kumar, Efficient fpga implementation of conjugate gradient methods for laplacian system using hls, arXiv preprint arXiv:1803.03797. 

\bibitem{kumar2016}
P. Kumar, L. Grigori, F. Nataf, and Q. Niu. On relaxed nested
factorization and combination preconditioning. International Journal of
Computer Mathematics, 93(1):179–199, 2016

\bibitem{kumar2010c}
P. Kumar, L. Grigori, Q. Niu, F. Nataf, Fourier Analysis of Modified Nested
Factorization Preconditioner for Three-Dimensional Isotropic Problems, HAL Id: inria-00448291, 2010. 

\bibitem{aggarwal2019}
A. Aggarwal, S. Kakkar, P. Kumar, Multithreaded Filtering Preconditioner for Diffusion Equation on Structured Grid, arXiv preprint arXiv:1909.09771, 2019.

\bibitem{kumar2015b}
P. Kumar, Fast solvers for nonsmooth optimization problems in phase separation, 2015 Federated Conference on Computer Science and Information Systems (FedCSIS), 2015. 

\bibitem{kumar2010d}
P. Kumar, A class of preconditioning techniques suitable for partial differential equations of structured and unstructured mesh, PhD thesis, 2010.

\bibitem{kumar2015e}
P. Kumar, Fast Preconditioned Solver for Truncated Saddle Point Problem in Nonsmooth Cahn–Hilliard Model, Recent Advances in Computational Optimization: Results of the Workshop on Computational Optimization WCO 2015

\bibitem{kumar2013}
P. Kumar, Stefano Markidis, Giovanni Lapenta, Karl Meerbergen,
and Dirk Roose. High performance solvers for implicit particle in cell
simulation. Procedia Computer Science, 18:2251–2258, 2013. 2013
International Conference on Computational Science

\bibitem{niu2010}
Q. Niu, L. Grigori, P. Kumar, F. Nataf, Modified tangential frequency filtering decomposition and its Fourier analysis, Numerische Mathematik 116 (1), 123-148, 2010.

\bibitem{kumar2011c}
P. Kumar, Purely algebraic domain decomposition methods for the incompressible Navier-Stokes equations, arXiv preprint arXiv:1104.3349, 2011. 

\bibitem{kumar2014s}
P. Kumar, Multithreaded direction preserving preconditioners, IEEE 13th International Symposium on Parallel and Distributed Computing, 2014.

\bibitem{kumar2013b}
P. Kumar, Karl Meerbergen, and Dirk Roose. Multi-threaded nested
filtering factorization preconditioner. PARA, Applied Parallel and Scientific Computing, pages 220–234,
Berlin, Heidelberg, 2013. 

\bibitem{saad2003}
Y. Saad. Iterative Methods for Sparse Linear Systems. Society for
Industrial and Applied Mathematics, Philadelphia, PA, USA, 2 edition,
2003.
\bibitem{han2023}
 A. Han, B. Mishra, P. Jawanpuria, P. Kumar, J. Gao. Riemannian Hamiltonian methods for min-max optimization on manifolds,SIAM J. of Optimization, Accepted, 2023.

\bibitem{naram2022}
J. Naram, T. Sinha, P. Kumar. A Riemannian Approach to Extreme Classification Problems,  CODS-COMAD, 2022

 
\end{thebibliography}

}
\end{document}